\documentclass[twoside,11pt]{article}

\usepackage{natbib}
\usepackage{times}
\usepackage{amsfonts}
\usepackage{amsmath}
\usepackage{amssymb}
\usepackage{latexsym}
\usepackage{color}
\usepackage{graphics}
\usepackage{enumerate}
\usepackage{amstext}
\usepackage{url}
\usepackage{epsfig}
\usepackage{tikz}
\usetikzlibrary{arrows,decorations.pathmorphing,backgrounds,positioning,fit,automata,shapes}
\usepackage{hyperref}

\usepackage{amsthm}

\numberwithin{equation}{section}
\newcounter{thm}
\theoremstyle{plain}
\newtheorem{theorem}{Theorem}
\theoremstyle{definition}
\newtheorem{definition}[thm]{Definition}
\theoremstyle{remark}
\newtheorem{rem}{Remark}
\theoremstyle{plain}
\newtheorem{fact}{Fact}

\newtheorem{proposition}{Proposition}
\newtheorem{lemma}[thm]{Lemma}
\newtheorem{corollary}[thm]{Corollary}

\global\long\def\mb#1{\mathbb{#1}}
\global\long\def\li#1#2{\lim_{#1\rightarrow#2}}

\global\long\def\lixz{\li x0}

\global\long\def\eps{\varepsilon}

\global\long\def\df{\triangleq}

\global\long\def\norm#1{\left\Vert #1\right\Vert }

\global\long\def\Sn{\mb{S}^{n-1}}

\global\long\def\E{\mathbb{E}}
\global\long\def\R{\mathbb{R}}

\global\long\def\N{\mathbb{N}}

\global\long\def\Rn{\mb{R}^n}

\global\long\def\e{\eps}
\global\long\def\FF{\mathcal{F}}

\global\long\def\mc#1{\mathcal{#1}}

\begin{document}

\title{The Sample Complexity of Dictionary Learning}

\author{
Daniel Vainsencher \\
\texttt{danielv@tx.technion.ac.il} \\ 
Department of Electrical Engineering \\ 
Technion, Israel Institute of Technology\\
Haifa 32000, Israel
\and
Shie Mannor \\
\texttt{shie@ee.technion.ac.il} \\ 
Department of Electrical Engineering \\ 
Technion, Israel Institute of Technology\\
Haifa 32000, Israel
\and
Alfred M. Bruckstein  \\
\texttt{freddy@cs.technion.ac.il} \\ 
Department of Computer Science\\ 
Technion, Israel Institute of Technology\\
Haifa 32000, Israel}
\maketitle

\begin{abstract} A large set of signals can sometimes be described sparsely using a dictionary, that is, every element can be represented as a linear combination of few elements from the dictionary.  
Algorithms for various signal processing applications, including classification, denoising and signal separation, learn a dictionary from a set of signals to be represented. Can we expect that the representation found by such a dictionary for a previously unseen example from the same source will have $L_2$ error of the same magnitude as those for the given examples?
We assume signals are generated from a fixed distribution, and study this questions from a statistical learning theory perspective. 

We develop generalization bounds on the quality of the learned dictionary for two types of constraints on the coefficient selection, as measured by the expected $L_2$ error in representation when the dictionary is used.
For the case of $l_1$ regularized coefficient selection we provide a generalization bound of the 
order of $O\left(\sqrt{np\log(m \lambda)/m}\right)$, where $n$ is the dimension, $p$ is the number of elements in the dictionary, $\lambda$ is a bound on the $l_1$ norm of the coefficient vector and $m$ is the number of samples, which complements existing results.
For the case of representing a new signal as a combination of at most $k$ dictionary elements, we provide a bound of
the order $O(\sqrt{np\log(m k)/m})$ under an assumption on the level of orthogonality of the dictionary (low Babel function).
We further show that this assumption holds for {\em most} dictionaries in high dimensions in a strong probabilistic sense.
Our results further yield fast rates of order $1/m$ as opposed to $1/\sqrt{m}$ using localized Rademacher complexity.
We provide similar results in a general setting using kernels with weak smoothness requirements.
\end{abstract}

\section{Introduction} 

In processing signals from $\mc{X}=\R^n$ it is now a common technique to use sparse representations; that is, to approximate each signal $x$ by a ``small'' linear combination $a$ of elements $d_i$ from a dictionary $D\in\mc{X}^p$, so that $x\approx Da = \sum_{i=1}^p a_id_i$. This has various uses detailed in Section~\ref{backgroundAndMotivation}. The smallness of $a$ is often measured using either $\norm{a}_1$, or the number of non zero elements in $a$, often 
denoted $\norm a_{0}$. The approximation error is measured here using a Euclidean norm appropriate to the vector space. We denote the approximation error of $x$ using dictionary $D$ and coefficients from $A$ as 
\begin{equation}\label{spRepError}
h_{A,D}(x)=\min_{a\in A}\norm{Da-x},
\end{equation}
where $A$ is one of the following sets determining the sparsity required of the representation:
 \[H_k=\left\{a:\norm{a}_0\le k\right\}\]
 induced a ``hard'' sparsity constraint, which we also call $k$ sparse representation, while
  \[R_\lambda=\left\{a:\norm{a}_1\le \lambda \right\}\] 
induces a convex constraint that is a ``relaxation'' of the previous constraint.

The dictionary learning problem is to find a dictionary $D$ minimizing 

\begin{equation}\label{eq:dictLearn}
E(D)=\E_{x\sim\nu}h_{A,D}(x),
\end{equation} 
where $\nu$ is a distribution over signals that is known to us only through samples from it. The problem addressed in this paper is the ``generalization'' (in the statistical learning sense) of dictionary learning: to what extent does the performance of a dictionary chosen based on a finite set of samples indicate its expected error in~\eqref{eq:dictLearn}? This clearly depends on the number of samples and other parameters of the problem such as dictionary size. In particular, an obvious algorithm is to represent each sample using itself, if the dictionary is allowed to be as large as the sample, but the performance on unseen signals is
likely to disappoint. 

To state our goal more quantitatively, assume that an algorithm finds a dictionary $D$ suited to $k$ sparse
representation, in the sense that the average representation error $E_m(D)$ on the
$m$ examples it is given is low. Our goal is to bound the generalization error $\e$, which
is the additional expected error that might be incurred:
$$E(D)\le (1+\eta)E_m(D) + \e,$$
where $\eta\ge 0$ is sometimes zero, and the bound depends on the number of samples and problem parameters. Since algorithms that find the optimal dictionary for a given set of samples (also known as
empirical risk minimization, or ERM, algorithms) are not known for dictionary learning, we prove
uniform convergence bounds that apply simultaneously over all
admissible dictionaries $D$, thus bounding from above the sample complexity of the dictionary learning problem.


Many analytic and algorithmic methods relying on the properties of finite dimensional Euclidean geometry can be applied in more general settings by applying kernel methods. These consist of treating objects that are not naturally represented in $\R^n$ as having their similarity described by an inner product in an abstract \emph{feature space} that is Euclidean. This allows the application of algorithms depending on the data only through a computation of inner products to such diverse objects as graphs, DNA sequences and text documents, that are not naturally represented using vector spaces~\citep{shawe2004kernel}. Is it possible to extend the usefulness of dictionary learning techniques to this setting? We address sample complexity aspects of this question as well.

\subsection{Background and related work} \label{backgroundAndMotivation}

Sparse representations are a standard practice in diverse fields such as signal processing, natural language processing, etc. Typically, the dictionary is assumed to be known. The motivation for sparse representations is indicated by the following results, in which we assume the signals come from $\mc{X}=\R^n$, and the representation coefficients from  $A=H_k$ where $k<n,p$ and typically $h_{A,D}(x)\ll 1$.

\begin{itemize}
\item Compression: If a signal $x$ has an approximate sparse
  representation in some commonly known dictionary $D$, then by
  definition, storing or transmitting the sparse representation will
  not cause large error.
\item Representation: If a signal $x$ has an approximate sparse
  representation in a dictionary $D$ that fulfills certain geometric
  conditions, then its sparse representation is unique and can be
  found efficiently~\citep{bruckstein2009sparse}.
\item Denoising: If a signal $x$ has a sparse representation in some
  known dictionary $D$, and $\tilde{x}=x+\nu$, where the random noise
  $\nu$ is Gaussian, then the sparse representation found for
  $\tilde{x}$ will likely be very close to $x$~\citep[for example][]{chen2001atomic}.
\item Compressed sensing: Assuming that a signal $x$ has a sparse
  representation in some known dictionary $D$ that fulfills certain geometric
  conditions, this representation can
  be approximately retrieved with high probability from a small number
  of random linear measurements of $x$. The number of measurements
  needed depends on the sparsity of $x$ in
  $D$~\citep{candes2006near}.
\end{itemize} 

The implications of these results are significant when a dictionary $D$ is
known that sparsely represents simultaneously  many signals.
In some applications the dictionary is chosen based on prior
knowledge, but in many applications the dictionary is learned based on
a finite set of examples. To motivate dictionary learning, consider
an image representation used for compression or denoising.  Different
types of images may have different properties (MRI images are not similar to scenery images), so that learning a specific dictionary to
each type of images may lead to improved performance. The benefits of dictionary learning have been demonstrated in many applications~\citep{protter2007sparse,peyre2009sparse,yang2009linear}.

Two extensively used techniques related to dictionary learning are Principal Component Analysis (PCA) and $k$ means clustering. The former finds a single subspace minimizing the sum of squared representation errors which is very similar to dictionary learning with $A=H_k$ and $p=k$. The latter finds a set of locations minimizing the sum of squared distances between each signal and the location closest to it which is very similar to dictionary learning with $A=H_1$ where $p$ is the number of locations. Thus we could see dictionary learning as PCA with multiple subspaces, or as clustering where multiple locations are used to represent each signal. The sample complexity of both algorithms are well studied~\citep{bartlett1998minimax,biau2008performance,shawe2005eigenspectrum,blanchard2007statistical}. 

This paper does not address questions of computational cost, though they are very relevant. Finding optimal coefficients for $k$ sparse representation (that is, minimizing~\eqref{spRepError} with $A=H_k$) is NP-hard in general~\citep{davis1997adaptive}. Dictionary learning as an optimization problem, that of minimizing ~\eqref{eq:dictLearn} is less well understood, even for empirical $\nu$ (consisting of a finite number of samples), despite over a decade of work on related algorithms with good empirical results~\citep{Olshausen97sparsecoding,Lewicki98learningovercomplete,kreutz2003dictionary,Aharon05k-svd:design,lee2007efficient,Krause_submodulardictionary,Mairal10}. 

The only prior work we are aware of that addresses generalization in dictionary learning, by \cite{maurer}, addresses the convex representation constraint $A=R_\lambda$; we discuss the relation of our work to theirs in Section~\ref{sec:res}. Another related work studies the
identifiability of dictionaries, giving conditions under which a
dictionary may be exactly recovered. A recent example giving somewhat
similar requirements on the number of samples (though in a different
setting, and to obtain a different kind of result)
is by~\cite{jaillet-l1}, which also includes a review of identifiability
results.

\section{Results}
\label{sec:res}


Except where we state otherwise, we assume signals are generated in the unit sphere $\mb{S}^{n-1}$.

\textbf{A new approach to dictionary learning generalization.}  Our
first main contribution is an approach to generalization bounds in
dictionary learning that is complementary to that used by \cite{maurer}.
Assume that the columns of the dictionary $D\in\R^{n\times p}$ are of unit
length, and that each signal $x\in \Sn$ is approximately represented
in the form $Da$ where the coefficient vector $a$ is known to
fulfill a constraint of form $\norm{a}_1\le \lambda$. 
We quantify the complexity of the associated error
function class in terms of $\lambda$, so that standard methods of uniform convergence give
generalization error bounds $\e$ of order $O\left(\sqrt{np\log(m \lambda)/m}\right)$ with $\eta=0$. The
method by~\cite{maurer} results in Theorem~\ref{thm:maurerDictLearning} given below providing generalization error bounds of order
$$O\left(\sqrt{p\min(p,n) \left(\lambda+\sqrt{\log(m\lambda)}\right)^2/m}\right).$$ Thus the latter are applicable to the case $n\gg p$, while our approach is not. However in the case $n<p$, also known in the literature as the ``over-complete''  case~\citep{Olshausen97sparsecoding,Lewicki98learningovercomplete}, the important complexity parameter is $\lambda$, on which our bounds depend only logarithmically, instead of polynomially. One case where this is significant is where the representation is chosen by solving a minimization problem such as $\min_a \norm{Da-X}+\gamma\cdot \norm{a}_1$ in which $\lambda=O\left(\gamma^{-1}\right)$.

\textbf{Fast rates.} For the case $\eta>0$ our methods are compatible with general fast rate methods of~\cite{BBMlocalizedrademacher05}, for bounds of order
$O(np\log(\lambda m)/m)$. The main significance of this is not in the
numerical results achieved, due to the large constants, but in that the general statistical behavior they imply
occurs in dictionary learning. For example, generalization error
has a ``proportional'' component which is reduced when the empirical
error is low. Whether fast rates results can be proved under the
infinite dimension regime is an interesting question we leave open. Note that
due to lower bounds by~\cite{bartlett1998minimax} of order $\sqrt{m^{-1}}$ on the $k$-means
clustering problem, which corresponds to dictionary learning for $1$-sparse representation, fast rates may be expected only with $\eta>0$, as presented here.

We now describe the relevant function class and the bounds on its complexity, which are proved in Section~\ref{sec:coverNumbers}, proving the following theorem The resulting generalization bounds are given explicitly at the end of this section.

\begin{theorem}\label{thm:coverNumbers}
  The function class
  $\mc{G}_\lambda=\left\{h_{R_\lambda,D}:\mb{S}^{n-1}\to \R :D\in \R^{n\times p},\norm{d_i}\le 1\right\}$, taken as a
  metric space with the metric induced by $\norm{\cdot}_\infty$, has an $\e$ cover of
  cardinality at most $\left(4\lambda/\e\right)^{np}$.  
\end{theorem}

\textbf{Extension to $k$ sparse representation.}
Our second main contribution is to extend both our approach and that of~\cite{maurer} to
 provide generalization bounds for dictionaries for $k$ sparse representations, by 
using a bound $\lambda$ on the $l_1$ norm of the representation coefficients 
when the dictionaries are close to orthogonal. 
Distance from orthogonality is measured by the Babel function, 
defined below and discussed in more detail in Section~\ref{sec:Babel}.

\begin{definition}[Babel function,~\citealt{Tropp04greedis}]
  For any $k\in \N$, the Babel function $\mu_{k}:\mc \R^{n\times
  m}\to\R^{+}$ is defined by:
\[ \mu_{k}\left(D\right)=\max\limits _{\Lambda\subset\left\{ 1,\dots
,p\right\} ;\left|\Lambda\right|=k}\max\limits _{i\notin
\Lambda}\sum_{\lambda\in\Lambda}\left|\left\langle
d_{\lambda},d_{i}\right\rangle \right|.\]
\end{definition}

The following proposition, which is proved in Section~\ref{sec:coverNumbers}, bounds the 1-norm of the dictionary coefficients for a $k$ sparse representation and also follows from analysis previously done by~\cite{donoho2003optimally,Tropp04greedis}.
\begin{proposition}\label{prop:bddCoefficients}
  Let $\norm{d_i} \in [1,\gamma]$ and $\mu_{k-1}\left(D\right)<1$,
  then a coefficient vector $a\in \R^p$ minimizing the $k$-sparse
  representation error $h_{H_k,D}(x)$ exists which has $\norm
  a_{1}\le \gamma k/\left(1-\mu_{k-1}\left(D\right)\right)$.
\end{proposition}

We now consider the class of all $k$ sparse representation error functions.
We prove in Section~\ref{sec:coverNumbers} the 
following bound on the complexity of this class.

\begin{corollary}\label{cor:kSparseCoverNumbers}
  The function class
  $\FF_{\delta,k}=\left\{h_{H_k,D}:\mb{S}^{n-1}\to \R :\mu_{k-1}(D)<\delta\right\}$, taken as a
  metric space with the metric induced by $\norm{\cdot}_\infty$, has an $\e$ cover of
  cardinality at most $\left(4k/\left(\e\left(1-\delta\right)\right)\right)^{np}$.
\end{corollary}

The dependence of the last two results on $\mu_{k-1}(D)$ means that the resulting 
bounds will
be meaningful only for algorithms which explicitly or implicitly
prefer near orthogonal dictionaries. Contrast this to Theorem~\ref{thm:coverNumbers} 
which has no significant conditions on the dictionary.

\textbf{Asymptotically almost all dictionaries are near orthogonal.}
A question that arises is
what values of $\mu_{k-1}$ can be expected for parameters $n,p,k$? We discuss
this question and prove the following probabilistic result in
Section~\ref{sec:Babel}.

\begin{theorem} \label{thm:probGood}Suppose that $D$ consist of $p$ vectors
chosen uniformly and independently from $\Sn$. Then we
have
 $$P\left(\mu_{k}>\frac{1}{2}\right)\le \frac{1}{\left(e^{\left(n-2\right)/\left(10k\log p\right)^{2}}-1\right)}.$$
\end{theorem}

Since low values of the Babel function have implications to representation finding algorithms, this result is of interest also outside the context of dictionary learning. Essentially it means that random dictionaries of size sub-exponential in $(n-2)/k^2$ have low Babel function.

\textbf{New generalization bounds for $l_1$ case.}
The covering number bound of Theorem~\ref{thm:coverNumbers} implies several generalization bounds for the problem of dictionary learning for $l_1$ regularized representation which we give here. These differ from those by~\cite{maurer} in depending more strongly on the dimension of the space, but less strongly on the particular regularization term. We first give the relevant specialization of the result by~\cite{maurer}
for comparison and for reference as we will later build on it. This result is independent of the dimension $n$ of the underlying space, thus the Euclidean unit ball $B$ may be that of a general Hilbert space, and the errors measured by $h_{A,D}$ are in the same norm.

\begin{theorem}[\citealt{maurer}]\label{thm:maurerDictLearning} Let $\max_{a\in A}\norm{a}_1\le \lambda$, and $\nu$ be any distribution on the unit sphere $B$. Then with probability at least $1-e^{-x}$ over the $m$ samples in $E_{m}$ drawn according to $\nu$, for all dictionaries $D\subset B$ with cardinality $p$:
\begin{align*} Eh^2_{A,D} & \le E_{m}h^2_{A,D}+\sqrt{\frac{p^2\left(14\lambda + 1/2\sqrt{\ln\left(16m\lambda^2\right)}\right)^2}{m}}+\sqrt{\frac{x}{2m}}.
\end{align*}
\end{theorem}

Using the covering number bound of Theorem~\ref{thm:coverNumbers} and a bounded differences concentration inequality (see Lemma~\ref{lem:slowRates}), we obtain the following result. The details are given in Section~\ref{sec:coverNumbers}.

\begin{theorem} \label{thm:lambdaGenSlow} Let $\lambda>0$, with $\nu$ a distribution on $\Sn$. Then with probability at least $1-e^{-x}$ over the $m$ samples in $E_{m}$ drawn according to $\nu$, for all $D$ with unit length columns:
\begin{align*} Eh_{R_\lambda,D} & \le E_{m}h_{R_\lambda,D}+\sqrt{\frac{np\ln\left(4\sqrt{m}\lambda\right)}{2m}}+\sqrt{\frac{x}{2m}}+\sqrt{\frac{4}{m}}.
\end{align*}
\end{theorem}

Using the same covering number bound and localized Rademacher complexity (see Lemma~\ref{lem:gen}), we obtain the following fast rates result. 
\begin{theorem} \label{thm:lambdaGen} Let $\lambda>0$, $K>1$, $\alpha>0$, with $\nu$ a distribution on $\Sn$. Then with probability at least $1-e^{-x}$ over the $m$ samples in $E_{m}$ drawn according to $\nu$, for all $D$ with unit length column:
\begin{align*} Eh_{R_\lambda,D} & \le \frac{K}{K-1}E_{m}h_{R_\lambda,D}+6K\max\left\{
\frac{8\alpha \lambda^{2}}{m},\left(480\right)^{2}\frac{\left(np+1\right)\log\left(\frac{m}{\alpha}\right)}{m},\frac{20+22\log\left(m\right)}{m}\right\}
\\ & +\frac{11x+5K}{m}.
\end{align*}
\end{theorem}

In any particular case, $\alpha$ and then $K$ may be chosen so as to
minimize the right hand side.

\textbf{Generalization bounds for $k$ sparse representation.}
Proposition~\ref{prop:bddCoefficients} and Corollary~\ref{cor:kSparseCoverNumbers} imply certain generalization bounds for the problem of dictionary learning for $k$ sparse representation, which we give here.
 
A straight forward combination of Theorem 2 of~\cite{maurer} (given here as Theorem~\ref{thm:maurerDictLearning}) and Proposition~\ref{prop:bddCoefficients} results in the following theorem.
\begin{theorem} \label{thm:genViaMaurer} Let $\delta<1$ with $\nu$ a distribution on $\Sn$. Then with probability at least $1-e^{-x}$ over the $m$ samples in $E_{m}$ drawn according to $\nu$, for all $D$ s.t. $\mu_{k-1}(D)\le \delta$:
$$Eh_{H_k,D}^{2}\le E_mh_{H_k,D}^{2}+\frac{p}{\sqrt{m}}\left(\frac{14k}{1-\delta }+\frac{1}{2}\sqrt{\ln\left(16m\left(\frac{k}{1-\delta}\right)^{2}\right)}\right)+\sqrt{\frac{x}{2m}}.$$
\end{theorem}

In the case of clustering we have $k=1$ and $\delta=0$ and this result approaches the rates of~\cite{biau2008performance}.

The following theorems follow from standard results and the covering number bound of Corollary~\ref{cor:kSparseCoverNumbers}.

\begin{theorem} \label{thm:deltaGenSlow} Let $\delta<1$ with $\nu$ a distribution on $\Sn$. Then with probability at least $1-e^{-x}$ over the $m$ samples in $E_{m}$ drawn according to $\nu$, for all $D$ s.t. $\mu_{k-1}(D)\le \delta$:
\begin{align*} Eh_{H_k,D} & \le E_{m}h_{H_k,D}+\sqrt{\frac{np\ln\frac{4\sqrt{m}k}{1-\delta}}{2m}}+\sqrt{\frac{x}{2m}}+\sqrt{\frac{4}{m}} .
\end{align*}
\end{theorem}

\begin{theorem} \label{thm:deltaGen} Let $\delta<1<K$, $\alpha>0$ with $\nu$ a distribution on $\Sn$. Then with probability at least $1-e^{-x}$ over the $m$ samples in $E_{m}$ drawn according to $\nu$, for all $D$ s.t. $\mu_{k-1}(D)\le \delta$:
\begin{align*} Eh_{H_k,D} & \le \frac{K}{K-1}E_{m}h_{H_k,D}+6K\max\left\{
\frac{8\alpha k^{2}}{m\left(1-\delta\right)^{2}},\left(480\right)^{2}\frac{\left(np+1\right)\log\left(\frac{m}{\alpha}\right)}{m},\frac{20+22\log\left(m\right)}{m}\right\}
\\ & +\frac{11x+5K}{m}.
\end{align*}
\end{theorem}

In any particular case, $\alpha$ and then $K$ may be chosen so as to
minimize the right hand side.

\textbf{Generalization bounds for dictionary learning in feature spaces.}
We further consider applications of dictionary learning to signals that are 
not represented as elements in a vector space, or that have a very high (possibly
 infinite) dimension.

 In addition to providing an approximate reconstruction of signals, sparse representation can also be considered as a form of analysis, if we treat the choice of non zero coefficients and their magnitude as features of the signal. In the domain of images, this has been used to perform classification (in particular, face recognition) by~\cite{wright2008robust}. Such analysis does not require that the data itself be represented in $\R^n$ (or in any vector space); it is enough that the similarity between data elements is induced from an inner product in a feature space. This requirement is fulfilled by using an appropriate kernel function.

\begin{definition}
Let $\mc{R}$ be a set of data representations, and let the kernel function $\kappa:\mc{R}^2\to \R$ and the feature mapping $\phi:\mc{R}\to \mc{H}$ be such that:
$$\kappa\left(x,y\right)=\left\langle \phi\left(x\right),\phi\left(y\right)\right\rangle$$
where $\mc{H}$ is some Hilbert space. 
\end{definition}

As a concrete example, choose a sequence of $n$ words, and let $\phi$ map a document to the vector of counts of appearances of each word in it (also called bag of words). Treating $\kappa(a,b)=\left<\phi(a),\phi(b) \right>$ as the similarity between documents $a$ and $b$, is the well known ``bag of words'' approach, applicable to many document related tasks~\citep{shawe2004kernel}. Then the statement $\phi(a)+\phi(b) \approx \phi(c)$ does not imply that $c$ can be reconstructed from $a$ and $b$, but we might consider it indicative of the content of $c$. The dictionary of elements used for representation could be decided via dictionary learning, and it is natural to choose the dictionary so that the bags of words of documents are approximated well by small linear combinations of those in the dictionary. 

As the example above suggests, the kernel dictionary learning problem is to find a dictionary $D$ minimizing
 $$\E_{x\sim\nu}h_{\phi,A,D}(x),$$ 
where we consider the representation error function
 $$h_{\phi,A,D}(x)=\min_{a\in A}\norm{\left(\Phi D\right)a-\phi\left(x\right)}_{\mc{H}},$$ 
in which $\Phi$ acts as $\phi$ on the elements of $D$, $A\in\left\{R_\lambda,H_k\right\}$, and the norm $\norm{\cdot}_\mc{H}$ is that induced by the kernel on the feature space $\mc{H}$.

Analogues of all the generalization bounds mentioned so far can be replicated in the kernel setting. The dimension free results of~\cite{maurer} apply most naturally in this setting, and may be combined with our results to cover also dictionaries for $k$ sparse representation, under reasonable assumptions on the kernel.

\begin{proposition}\label{prop:dimensionfreeKGen}
Let $\nu$ be any distribution on $\mc{R}$ such that when $x\sim \nu$ we have $\norm{\phi(x)}\le 1$ with probability 1. Then with probability at least $1-e^{-x}$ over the $m$ samples in $E_{m}$ drawn according to $\nu$, for all $D\subset \mc{R}$ with cardinality $p$ such that $\Phi D\subset B_{\mc H}$ and $\mu_{k-1}(\Phi D)\le \delta < 1$:
\begin{align*} Eh^2_{\phi,H_k,D} & \le E_{m}h^2_{\phi,H_k,D}+\sqrt{\frac{p^2\left(14k/(1-\delta) + 1/2\sqrt{\ln\left(16m\left(\frac{k}{1-\delta}\right)^2\right)}\right)^2}{m}}+\sqrt{\frac{x}{2m}}.
\end{align*}
\end{proposition}

Note that the Babel function is defined in terms of inner products between elements of $D$, and can therefore be computed in $\mc{H}$ by applications of the kernel. 

 This result is proved in Section~\ref{sec:kernel}, as well as the cover number bounds (using some additional definitions and assumptions described there) that are used to prove the remaining generalization bounds, of which one is given below. 

\begin{theorem} \label{thm:kernelGen} Let $\mc{R}$ have $\e$ covers of order $\left(C/\e\right)^n$. Let $\kappa:\mc{R}^2 \to \R^+$ be a kernel function s.t. $\kappa(x,y)=\left<\phi(X),\phi(Y)\right>$, for $\phi$ which is uniformly $L$-H\"older of order $\alpha>0$ over $\mc{R}$, and let $\gamma=\max_{x\in\mc{R}}\norm{\phi(x)}_\mc{H}$. Let $\delta<1$, and $\nu$ any distribution on $\mc{R}$, then with probability at least $1-e^{-x}$ over the $m$ samples in $E_{m}$ drawn according to $\nu$, for all dictionaries $D\subset \mc{R}$ of cardinality $p$ s.t. $\mu_{k-1}(\Phi D)\le \delta < 1$ (where $\Phi$ acts like $\phi$ on columns):
\begin{align*} Eh_{H_k,D} & \le E_{m}h_{H_k,D}+\gamma\left(\sqrt{\frac{np\ln\left( \sqrt{m} C^\alpha \frac{k\gamma^2 L}{1-\delta}\right)}{2\alpha m}}+\sqrt{\frac{x}{2m}}\right)+\sqrt{\frac{4}{m}}.
\end{align*}
\end{theorem}

The covering number bounds needed to prove this theorem and analogs for the other generalization bounds are proved in Section~\ref{sec:kernel}.

\section{Covering numbers of $\mc{G}_\lambda$ and $\FF_{\delta,k}$}\label{sec:coverNumbers}

The main content of this section is the proof of Theorem ~\ref{cor:kSparseCoverNumbers} and Corollary~\ref{cor:kSparseCoverNumbers}. We also show that the restriction of near-orthogonality on the set of dictionaries, on which we rely in the proof for $k$ sparse representation, is necessary to achieve a bound on $\lambda$. Lastly, we recall known results from statistical learning theory that link covering numbers to generalization bounds.

We recall the definition of the covering numbers we wish to bound. \cite{anthony1999neural} give a textbook introduction to covering numbers and their application to generalization bounds.

\begin{definition}[Covering number] Let $\left(M,d\right)$ be a metric
space and $S\subset M$. Then the $\e$ covering number of $S$ defined as
$N\left(\e,S,d\right)=\min\left\{ \left|A\right||A\subset
M \mbox{ and } S\subset\left(\bigcup_{a\in A}B_{d}\left(a,\e\right)\right)\right\} $ is the size of the minimal $\e$ cover of $S$ using $d$.
\end{definition}

 To prove Theorem~\ref{thm:coverNumbers} and
Corollary~\ref{cor:kSparseCoverNumbers} we first note that the space
of all possible dictionaries is a subset of a unit ball in a Banach
space of dimension $np$ (with a norm specified below). Thus by
proposition 5 formalized by~\cite{cucker2002mathematical} the space of
dictionaries has an $\e$ cover of size $\left(4/\e\right)^{np}$. We
also note that a uniformly $L$ Lipschitz mapping between metric spaces
converts $\e/L$ covers into $\e$ covers. Then it is enough to show
that $\Psi_\lambda$ defined as $D\mapsto h_{R_\lambda,D}$ and $\Phi_k$
defined as $D\mapsto h_{H_k,D}$ are uniformly Lipschitz (when $\Phi_k$
is restricted to the dictionaries with $\mu_{k-1}(D)\le c<1$). 
The proof of these Lipschitz properties is our next goal, in the form of Lemmas~\ref{lem:psiLip} and~\ref{lem:phiLip}.

The first step is to be clear about the metrics we consider over the
spaces of dictionaries and of error functions. We start by defining 
the following norm.

\begin{definition} Let $D\in \R^{n\times p}$. We denote
$\norm{D}_{ME}=\max_{i}\norm{d_{i}}$ the norm of its maximal
column.
\end{definition}

We will use the fact $\norm{\cdot}_{ME}$ upper bounds
a certain induced norm.

\begin{definition}[Induced matrix norm] Let $p,q\in\N$, then a matrix
$A\in\R^{n\times m}$ can be considered as an operator
$A:\left(\R^{m},\norm{\cdot}_{p}\right)\to\left(\R^{n},\norm{\cdot}_{q}\right)$.
Then the $p,q$ induced norm is defined as $\norm
A_{p,q}\df\sup_{x\in\R^{m}\norm x_{p}=1}\norm{Ax}_{q}$.
\end{definition}

\begin{fact} $\norm{D}_{1,2}\le \norm{D}_{ME}$
\end{fact}

The geometric interpretation of this fact is that $Da/\norm{a}_1$ is a
convex combination of vectors each of length at most $\norm{D}_{ME}$,
then $\norm{Da}_2\le \norm{D}_{ME}\norm{a}_1$.

The images of $\Psi_\lambda$ and $\Phi_k$ are sets of representation error
functions--each dictionary induces a set of precisely representable
signals, and a representation error function is simply a map of
distances from this set. Representation error functions are clearly
continuous, 1-Lipschitz, and into $[0,1]$. In this setting, a natural norm over the images is the
supremum norm $\norm{\cdot}_\infty$.

\begin{lemma}\label{lem:psiLip} The function $\Psi_\lambda$ is $\lambda$-Lipschitz from
$\left(\R^{n\times m},\norm{\cdot}_{ME}\right)$ to
$C\left(\Sn\right)$.
\end{lemma}

\begin{proof} Let $D$ and $D'$ be two normalized dictionaries whose
corresponding elements are at most $\e>0$ far from one another. Let $x$
be a unit signal and $Da$ an optimal representation for it.  Then
$\norm{\left(D-D'\right)a}\le \norm{D-D'}_{1,2}\norm{a}_1 \le
\norm{D-D'}_{ME}\norm{a}_1 \le \e\lambda$. Then $g_{\lambda,D'}(x)\le
g_{\lambda,D}(x)+\e\lambda$ and by symmetry we have
$\left|\Psi_\lambda(D)(x)-\Psi_\lambda(D')(x)\right|\le \lambda\e$. This holds
for all unit signals, then $\norm{\Psi_\lambda(D)-\Psi_\lambda(D')}_\infty \le \lambda\e$.
\end{proof}

We now provide a proof for Proposition~\ref{prop:bddCoefficients} which is used in the corresponding treatment for covering numbers under $k$ sparsity.

\begin{proof}[Of Proposition~\ref{prop:bddCoefficients}]
  Assume that $\mu_{k-1}(D)\le \delta < 1 \le \min_{i\le p}\norm{d_i}_2\le \gamma $. Let $D^k$ be a set of $k$ elements from $D$ achieving the minimum on $h_{H_k,D}(x)$, with $x\in \Sn$. We now consider the Gram matrix $G=\left(D^{k}\right)^{\top}D^{k}$. The matrix $G$ is symmetric, therefore it scales each point in the unit sphere by a non-negative combination of its real eigenvalues. Also, the diagonal entries of $G$ are the norms of the elements of $D^k$, therefore at least 1. By the Gersgorin theorem~\citep{horn1990matrix}, the eigenvalues of the Gram matrix are lower bounded by $1-\delta>0$. Then in particular $G$ has a symmetric inverse, which scales each point by no more than $1/(1-\delta)$. Then $\norm{G^{-1}}_{1,1}\le 1/(1-\delta)$.

In particular, elements of $D^k$ are linearly independent, which implies that the unique optimal representation of $x$ as a linear combination of the columns of $D^k$ is $D^ka$ with 
$$ a = \left(\left(D^{k}\right)^{\top}D^{k}\right)^{-1}\left(D^{k}\right)^{\top}x.$$
By the definition of induced matrix norms, we have 
$ \norm{a}_{1}\le\norm{\left(\left(D^{k}\right)^{\top}D^{k}\right)^{-1}}_{1,1}\norm{\left(D^{k}\right)^{\top}x}_{1}\le \gamma k/(1-\delta)$, the last bound because $x$ is a unit vector,  and $D^k$ has $k$ columns whose norm is bounded by $\gamma$.
\end{proof}

\begin{lemma}\label{lem:phiLip} The function
$\Phi_k$ is a $k/(1-\delta)$-Lipschitz mapping from the set of
normalized dictionaries with $\mu_{k-1}(D)<\delta$ with the metric
induced by $\norm{\cdot}_{ME}$ to $C\left(\Sn\right)$.
\end{lemma}

The proof of this lemma is the same as that of
Lemma~\ref{lem:psiLip}, except that $a$ is taken to be an optimal
representation that fulfills $\norm{a}_1\le \lambda = k/\left(1-\mu_{k-1}(D)\right)$,
whose existence is guaranteed by Proposition~\ref{prop:bddCoefficients}.

This concludes the proof of Theorem~\ref{thm:coverNumbers} and Corollary~\ref{cor:kSparseCoverNumbers}.

The next theorem shows that unfortunately, $\Phi$ is {\em not} uniformly $L$-Lipschitz for any constant $L$, requiring its restriction to an
appropriate subset of the dictionaries.

\begin{theorem}\label{thm:notLipschitz} For any $k,n,p$, there exists $c>0$ and $q$, such that for every $\e>0$, there exist $D,D'$ such
that $\norm{D-D'}_{ME}<\e$ but
$\left|\left(h_{H_k,D}(q)-h_{H_k,D'}(q)\right)\right|>c$.
\end{theorem}

\begin{proof} First we show that there exists $c>0$ such that every
dictionary will have $k$ sparse representation error of at least $c$
on some signal.  Let $\nu_{S^{n-1}}$ be the uniform probability measure on
the sphere, and $A_c$ the probability assigned by it to the set
within $c$ of a $k$ dimensional subspace. As $c\searrow 0$,
$A_c$ also tends to zero, then there exists $c>0$
s.t. $\binom{p}{k}A_c<1$. Then for that $c$ there exists a set of positive measure on which $h_{H_k,D}>c$, let $q$ be a point in this set.

To complete the proof we consider a dictionary $D$ whose first $k-1$
elements are the standard basis $\left\{e_1,\dots,e_{k-1}\right\}$,
its $k$ the element is $D_k=\sqrt{1-\e^2/2}e_1+\e e_k/2$, and the
remaining elements are chosen arbitrarily. 
Now construct $D'$ to be identical to $D$ except its
$k$th element is $v=\sqrt{1-\e^2/2} e_1+l q$ choosing $l$ so that
$\norm{v}_2=1$ (which implies that $\left|l\right|<\e/2$).  Then
$\norm{D-D'}_{ME}=\norm{\e e_k/2+l q}_2\le \e$ and $h_{H_k,D'}(q)=0$.
\end{proof}

To conclude the generalization bounds of Theorems~\ref{thm:lambdaGenSlow},~\ref{thm:lambdaGen},~\ref{thm:deltaGenSlow}, ~\ref{thm:deltaGen} and~\ref{thm:kernelGen} from the covering number bounds we have provided, we use the following two results. The first has a simple proof which we therefore give here. The second result is an adaptation of results by~\cite{BBMlocalizedrademacher05}, 
to our needs, and explained further in the appendix. 

\begin{lemma}\label{lem:slowRates}
Let $\FF$ be a  class of $[0,B]$ functions with covering number bound $\left(C/\e\right)^d>e/B^2$ under the supremum norm.
Then for every $x>0$, with probability of at least $1-e^{-x}$ over the $m$ samples in $E_m$ chosen according to $\nu$, for all $f\in \FF$:
$$Ef\le E_{m}f+B\left(\sqrt{\frac{d\log \left(C\sqrt{m}\right)}{2m}}+\sqrt{\frac{x}{2m}}\right)+\sqrt{\frac{4}{m}} .$$
\end{lemma}

\begin{lemma}\label{lem:gen}
If $\mc F$ is a class of $\left[0,1\right]$ functions with $C>2$
and $d\in\N$ s.t. $N\left(\e,\FF,L_2(\nu)\right)\le\left(\frac{C}{\e}\right)^{d}$
for every probability measure $\nu$ and $\e>0$, then for all $K,\alpha,x>0,f\in\FF$,
with probability at least $1-e^{-x}$ over the $m$ samples used in $E_{m}$ and drawn from $\nu$:
\[
Ef\le\frac{K}{K-1}E_{m}f+6K\max\left\{ \frac{\alpha C^{2}}{2m},\left(480\right)^{2}\frac{\left(d+1\right)\log\left(\frac{m}{\alpha}\right)}{m},\frac{20+22\log\left(m\right)}{m}\right\} +\frac{11x+5K}{m}.\]
\end{lemma}

Our fast rates results are simple applications of this lemma, noting that an $\e$ cover in $C\left(\Sn\right)$ is also an $\e$ cover under an $L_2$ metric induced by any measure.

\begin{proof}[Of Lemma~\ref{lem:slowRates}]
We wish to bound $\sup_{f\in\FF}Ef-E_mf$. Take $\FF_\e$ to be a minimal $\e$ cover of $\FF$, then for an arbitrary $f$, denoting $f_\e$ an $\e$ close member of $\FF_\e$, $Ef-E_mf\le Ef_\e -E_mf_\e+2\e$. In particular,
 $\sup_{f\in\FF}Ef-E_mf\le 2\e + \sup_{f\in\FF_\e}Ef-E_mf$.
To bound the supremum on the now finite class of functions, note that $Ef-E_mf$ is a function of $m$ independent variables (the samples chosen according to $\nu$), which changes by at most $B/m$ when one of the variables is modified. Then by the bounded differences inequality, $P\left(Ef-E_mf-\E(Ef-E_mf) > t\right) = P\left(Ef-\E_mf > t\right) \le \exp\left(-2mB^{-2}t^2\right)$.

The probability that any of the $\left|\FF_\e\right|$ differences under the supremum is larger than $t$ may be union bounded as $\left|\FF_\e\right|\cdot \exp\left(-2mB^{-2}t^2\right) \le \exp\left(d\log\left(C/\e\right)-2mB^{-2}t^2\right)$.

In order to control the probability with $x$ as in the statement of the lemma, we need to have $x=d\log\left(C/\e\right)-mB^{-2}t^2$ and thus we choose $t=\sqrt{B^2/2m}\sqrt{d\log\left(C/\e\right)+x}$. Then with high probability we bound the supremum of differences by $t$ which is upper bounded, using the assumption on the covering number bound, by $B\left(\sqrt{d\log\left(C/\e\right)/2m} +\sqrt{x/2m}\right)$.

Then the proof is completed by substitution into the bound over the whole function class $\mc{F}$ and taking $\e=1/\sqrt{m}$.
\end{proof}

\section{On the Babel function}
\label{sec:Babel}

The Babel function is one of several metrics defined in the sparse
representations literature to quantify an "almost orthogonality"
property that dictionaries may enjoy. Such properties have been
shown to imply theoretical properties such as uniqueness of the
optimal $k$ sparse representation. In the algorithmic context, 
\cite{donoho2003optimally} and \cite{Tropp04greedis}
use the Babel function to show that particular tractable algorithms
for finding sparse representations are indeed approximation algorithms
when applied to such dictionaries. This reinforces the practical
importance of the learnability of this class of dictionary. We proceed
to discuss some elementary properties of the Babel function, and then
state a bound on the proportion of
dictionaries having sufficiently good Babel function.

Measures of orthogonality are typically defined in terms of inner
products between the elements of the dictionary. Perhaps the simplest
of these measures of orthogonality is the following special case of
the Babel function.
\begin{definition} The coherence of a dictionary $D$ is
$\mu_1(D)=\max_{i,j}\left|\left<d_i,d_j \right>\right|$.
\end{definition} The Babel function, in considering sums of $k$ inner
products at a time, rather than the maximum over all inner products,
is better adapted to quantify the effects of non orthogonality on
representing a signal with particular level $k+1$ of sparsity. The
additional expressive power of $\mu_k$ over $\mu_1$ is illustrated by
considering that ensuring that $\mu_k<1$ by restricting $\mu_1$
implies the constraint $\mu_1(D)<1/k$, which for $k>1$ would exclude a
dictionary in which pairs of elements have inner product 0 except for
some disjoint pairs whose inner product equals to half, despite such a
dictionary having $\mu_k=1/2$ for any $k$.

To better understand $\mu_{k}\left(D\right)$, we consider first its
extreme values. When $\mu_{k}\left(D\right)=0$, for any $k>1$, this
means that $D$ is an orthogonal set (therefore $p\le n$). The maximal
value of $\mu_{k}\left(D\right)$ is $k$, and occurs only if some
dictionary element is repeated (up to sign) at least $k+1$ times.

A well known generic class of dictionaries with more elements than a
basis is that of \emph{frames}~\citep[see][]{dun1952class}, which include many wavelet systems and filter
banks. Some frames can be trivially seen to fulfill our condition on the
Babel function.
\begin{proposition} Let $D\in \R^{n\times p}$ be a frame of $\Rn$, so
that for every $v\in \Sn$ we have that $A \le
\sum_{i=1}^n\left|\left<v,d_i\right>\right| \le B$, with $\norm{d_i}_2
= 1$ for all $i$, and $B<1+1/(p-1)$. Then $\mu_{k-1}(D)<1$.
\end{proposition}

This may be easily verified using the relation between
$\norm{\cdot}_1$ and $\norm{\cdot}_2$ in $\R^{p-1}$.

\subsection{Proportion of dictionaries with $\mu_{k-1}(D)<\delta$}

We return to the question of the prevalence of dictionaries from
$D_\delta$. Are almost all dictionaries in $D_\delta$? If the answer is affirmative, it
implies that 
Theorem~\ref{thm:deltaGen} is quite strong, and representation finding
algorithms such as basis pursuit are almost always exact, which might
help prove properties of dictionary learning algorithms. If the
opposite is true and few dictionaries are in $D_\delta$, the
results of this paper are  weak. While there might be better measures 
on the space of dictionaries, we consider one that seems natural: suppose that a 
dictionary $D$ is constructed by choosing $p$ unit
vectors uniformly from $\Sn$; what is the probability that
$\mu_{k-1}(D)<\delta$?

Theorem~\ref{thm:probGood} gives us the following answer to this question. Under the assumption that the sparsity
parameter $k$ grows slowly, if at all, as $n\nearrow\infty$
(specifically, that $k\log p = o(\sqrt{n})$), this theorem 
 implies that asymptotically  {\em almost all
dictionaries under the Lebesgue measure are learnable}.  

The remainder of this section is devoted to the proof of Theorem~\ref{thm:probGood}. 
This proof relies
heavily on the Orlicz norms for random variables
and their properties; \cite{van1996weak} give a detailed introduction. We recall a few of the definitions and facts presented there.

\begin{definition} Let $\psi$ be a non-decreasing, convex function with
$\psi(0)=0$, and let $X$ be a random variable. Then
 $$\norm{X}_\psi =
\inf\left\{C>0:\E\psi\left(\frac{\left|X\right|}{C}\right)<1\right\}$$
is called an Orlicz norm.
\end{definition} 

As may be verified, these are indeed norms for
appropriate $\psi$, such as $\psi_2\df e^{x^2}-1$, which is the case that
will interest us most.

By the Markov inequality we can obtain that variables with finite
Orlicz norms have light tails.
\begin{fact} We have 
$P\left(\left|X\right|>x\right)\le\left(\psi_{2}\left(x/\norm
X_{\psi_{2}}\right)\right)^{-1}.$
\end{fact}

The next fact is an almost converse to
the last fact, stating that light tailed random variables have finite
$\psi_2$ Orlicz norms.

\begin{fact}\label{fact:lightTail} Let $A,B>0$ and $P\left(\left|X\right|\ge
x\right)\le Ae^{-Bx^{2}}$ for all x, where $p\ge1$, then $\norm
X_{\psi_{2}}\le\left(\left(1+A\right)/B\right)^{1/2}$.
\end{fact}

The following bound on the maximum of variables with
light tails.

\begin{fact}\label{fact:extreme} We have $\norm{\max_{1\le i\le
m}X_{i}}_{\psi_{2}}\le K\sqrt{\log m}\max_{i}\norm{X_{i}}_{\psi_{2}}.$
\end{fact}

The constant $K$ may be upper bounded by $\sqrt{2}$. Note that the independence of $X_i$ is not required.

We use also one isoperimetric fact about the sphere in high dimension.

\begin{definition} The $\e$ expansion of a set $D$ in a metric space
$\left(X,d\right)$ is defined as $$D_{\e}=\left\{ x\in
X|d\left(x,D\right)\le\e\right\},$$ where $d(x,A)=\inf_{a\in A}d(x,a)$.
\end{definition}

\begin{fact}[L\'evy's isoperimetric
inequality~\citeyear{levy1952problemes}]\label{fact:isoperimetric}

Let $C$ be one half of $\Sn$, then
$\mu\left(\left(\Sn\backslash C_{\e}\right)\right)\le\sqrt{\frac{\pi}{8}}\exp\left(-\frac{(n-2)\e^{2}}{2}\right)$.
\end{fact}

Our goal in the reminder of this subsection is to obtain the following
bound.

\begin{lemma}\label{lem:psiBound} Let $D$ be a dictionary chosen at
random as described above, then
 $$\norm{\mu_k(D)}_{\psi_2}\le
5k\log p/\sqrt{n-2}.$$
\end{lemma}

Our probabilistic bound on $\mu_{k-1}$ is a direct conclusion of Fact~\ref{fact:lightTail} and Lemma~\ref{lem:psiBound} which we now proceed to prove. 
The plan of our proof is to bound the $\psi_2$ metric of $\mu_k$ from
the inside terms and outward using Fact~\ref{fact:extreme} to
overcome the maxima over possibly dependent random variables.

\begin{lemma} Let $X_1,X_2$ be unit vectors chosen uniformly and
independently from $\Sn$, then $$\norm{\left|\left\langle
X_1,X_2\right\rangle
\right|}_{\psi_{2}}\le\sqrt{6/(n-2)}.$$ 
\end{lemma}
We denote the bound on the right hand side $W$.
\begin{proof} Taking $X$ to be uniformly chosen from $\Sn$, for any
constant unit vector $x_0$ we have that $\left<X,x_0\right>$ is a
light tailed random variable by Fact~\ref{fact:isoperimetric}.  By
Fact~\ref{fact:lightTail}, we may bound
$\norm{\left<X,x_0\right>}_{\psi_2}$. Replacing $x_0$ by
a random unit vector is equivalent to applying to $X$ a uniformly
chosen rotation, which does not change the analysis.
\end{proof}

The next step is to bound the inner maximum appearing in the
definition of $\mu_k$.
\begin{lemma} Let $\left\{d_i\right\}_{i=1}^p$ be uniformly
and independently chosen unit vectors
then $$\norm{\max_{\Lambda\subset\left\{ 2,\dots, p\right\}
\wedge\left|\Lambda\right|=k}\sum_{\lambda\in\Lambda}\left|\left\langle
d_1,d_{\lambda}\right\rangle \right|}_{\psi_{2}}\le
kKW\sqrt{\log\left(p-1\right)}.$$
\end{lemma}

\begin{proof} Take $X_{\lambda}$ to be $\left\langle
D_1,D_{\lambda}\right\rangle$. Then using Fact~\ref{fact:extreme} and the previous lemma we find 
 $$\norm{\max_{1\le\lambda\le
p\wedge\lambda\ne i}\left|X_{\lambda}\right|}_{\psi_{2}}\le
K\sqrt{\log\left(p-1\right)}\max_{\lambda}\norm{\left|X_{\lambda}\right|}_{\psi_{2}}\le
KW\sqrt{\log\left(p-1\right)}.$$

Define the random permutation $\lambda_{j}$
s.t. $\left|X_{\lambda_{j}}\right|$ are non-increasing. In this notation, it is clear that  
$\max_{\Lambda\subset\left\{ 2\dots p\right\}
\wedge\left|\Lambda\right|=k}\sum_{\lambda\in\Lambda}\left|X_{\lambda}\right|=\sum_{j=1}^{k}\left|X_{\lambda_{j}}\right|$. Note
that $\left|X_{\lambda_{i}}\right|\le\left|X_{\lambda_{1}}\right|$
then for every i,
$\norm{\left|X_{\lambda_{i}}\right|}_{\psi_{2}}\le\norm{\left|X_{\lambda_{1}}\right|}_{\psi_{2}}\le
KW\sqrt{\log\left(p-1\right)}$.

By the triangle inequality,
$\norm{\sum_{j=1}^{m}\left|X_{\lambda_{j}}\right|}_{\psi_{2}}\le\sum_{j=1}^{m}\norm{\left|X_{\lambda_{j}}\right|}_{\psi_{2}}\le
mKW\sqrt{\log\left(p-1\right)}$.
\end{proof}

\begin{rem}
\rm{Two facts are relevant to the tightness of the approximations in the
last proof. First, that $|X_{\lambda_i}|$ are variables with positive
expectation bounded away from zero, thus the norm of their sum must
scale at least linearly in the number of summands, so the triangle
inequality is essentially tight. Second we consider the bound
$\norm{\left|X_{\lambda_{i}}\right|}_{\psi_{2}}\le\norm{\left|X_{\lambda_{1}}\right|}_{\psi_{2}}$,
and note its looseness is strictly limited by the slow growth of
$\sqrt{\log(\cdot)}$, and in any case is bounded by 2.  }
\end{rem}

To complete
the proof of Lemma~\ref{lem:psiBound}, we replace $D_1$ with the dictionary element maximizing the Orlicz norm, by another application of Fact~\ref{fact:extreme}, and to complete the proof of Theorem~\ref{thm:probGood}, apply Fact~\ref{fact:lightTail} to the estimated Orlicz norm.

\section{Dictionary learning in feature spaces}
\label{sec:kernel}
We propose in Section~\ref{sec:res} a scenario in which dictionary learning is performed in a feature space corresponding to a kernel function. Here we show how to adapt the different generalization bounds discussed in this paper for the particular case of $\R^n$ to more general feature spaces, and the dependence of the sample complexities on the properties of the kernel function or the corresponding feature mapping. We begin with the relevant specialization of the results of \cite{maurer} which have the simplest dependence on the kernel, and then discuss the extensions to $k$ sparse representation and to the cover number techniques presented in the current work.

Theorem~\ref{thm:maurerDictLearning} applies as is to the feature space, under the simple assumption that the dictionary elements and signals are in its unit ball which is guaranteed by some kernels such as the Gaussian kernel. Then we take $\nu$ on the unit ball of $\mc{H}$ to be induced by some distribution $\nu'$ on the domain of the kernel, and the theorem applies to any such $\nu'$ on $\mc{R}$. Nothing more is required if the representation is chosen from $R_\lambda$. The corresponding generalization bound for $k$ sparse representations when the dictionary elements are near orthogonal in the feature space is given in Proposition~\ref{prop:dimensionfreeKGen}.

\begin{proof}[Of Proposition~\ref{prop:dimensionfreeKGen}]
Proposition~\ref{prop:bddCoefficients} applies with the Euclidean norm of $\mc{H}$, and $\gamma=1$. We apply Theorem~\ref{thm:maurerDictLearning} with $\lambda=k/\left(1-\delta \right)$.
\end{proof}

The results so far show that generalization in dictionary learning can occur despite the potentially infinite dimension of the feature space, without considering practical issues of representation and computation. We now make the domain and applications of the kernel explicit in order to address a basic computational question, and allow the use of cover number based generalization bounds to prove Theorem~\ref{thm:kernelGen}. We now consider signals represented in a metric space $\left(\mc{R},d\right)$, in which similarity is measured by the kernel $\kappa$ corresponding to the feature map $\phi:\mc{R}\to\mc{H}$. The elements of a dictionary $D$ are now from $\mc{R}$, and we denote $\Phi D$ their mapping by $\phi$ to $\mc{H}$. Then representation error function used is $h_{\phi,A,D}$.

We now show that the approximation error in feature space is a quadratic function of the coefficient vector, which may be found by applications of the kernel.

\begin{proposition}
Computing the representation error at a given $x,a,D$ requires $O\left(p^{2}\right)$
kernel applications in general, and only $O\left(k^{2}+p\right)$ when
$a$ is $k$ sparse.
\end{proposition}
\begin{proof} 
Writing the error;
\begin{align*}
\norm{\left(\Phi D\right)a-\phi\left(x\right)}^{2} & =\left\langle \left(\Phi D\right)a-\phi\left(x\right),\left(\Phi D\right)a-\phi\left(x\right)\right\rangle \\
 & =\left\langle \left(\Phi D\right)a,\left(\Phi D\right)a\right\rangle +\left\langle \phi\left(x\right),\phi\left(x\right)\right\rangle -2\left\langle \phi\left(x\right),\left(\Phi D\right)a\right\rangle \\
 & =\left\langle \sum_{i=1}^{p}\phi\left(d_{i}\right)a_{i},\sum_{j=1}^{p}\phi\left(d_{j}\right)a_{j}\right\rangle +\left\langle \phi\left(x\right),\phi\left(x\right)\right\rangle -2\left\langle \phi\left(x\right),\sum_{i=1}^{p}\phi\left(d_{i}\right)a_{i}\right\rangle \\
 & =\sum_{i=1}^{p}a_{i}\sum_{j=1}^{p}a_{j}\left\langle \phi\left(d_{i}\right),\phi\left(d_{j}\right)\right\rangle +\left\langle \phi\left(x\right),\phi\left(x\right)\right\rangle -2\sum_{i=1}^{p}a_{i}\left\langle \phi\left(x\right),\phi\left(d_{i}\right)\right\rangle \\
 & =\sum_{i=1}^{p}a_{i}\sum_{j=1}^{p}a_{j}\kappa\left(d_{i},d_{j}\right)+\kappa\left(x,x\right)-2\sum_{i=1}^{p}a_{i}\kappa\left(x,d_{i}\right).\end{align*}
\end{proof}
We note that the $k$ sparsity constraint on $a$ poses algorithmic difficulties beyond those addressed here. Some of the common approaches to these, such as Orthogonal Matching Pursuit~\citep{chen1989orthogonal}, also depend on the data only through their inner products, and may therefore be adapted to the kernel setting.

The cover number bounds depend strongly on the dimension of the space of dictionary elements. Taking $\mc{H}$ as the space of dictionary elements is the simplest approach, but may lead to vacuous or weak bounds, for example in the case of the Gaussian kernel whose feature space is infinite dimensional. Instead we propose to use the space of data representations $\mc{R}$, whose dimensions are generally bounded by practical considerations. In addition, we will assume that the kernel is not ``too wild'' in the following sense.

\begin{definition}
Let $L,\alpha>0$, and let $(A,d')$ and $(B,d)$ are metric spaces. We say a mapping $f:A\to B$ is uniformly $L$ H\"older of order $\alpha$ on a set $S\subset A$ if $\forall x,y\in S$, the following bound holds:
 $$d\left(f(x),f(y)\right)\le L\cdot d'(x,y)^\alpha.$$
\end{definition}

The relevance of this smoothness condition is as follows: 

\begin{fact}
A H\"older function maps an $\e$ cover of $S$ to an $L\e^\alpha$ cover of its image $f(S)$. Thus, to obtain an $\e$ cover of the image of $S$, it is enough to begin with an $\left(\e/L\right)^{1/\alpha}$ cover of $S$. 
\end{fact}

A H\"older feature map $\phi$ allows us to bound the cover numbers of the dictionary elements in $\mc{H}$ using their cover number bounds in $\mc{R}$. Note that not every kernel corresponds to a H\"older feature map (the Dirac $\delta$ kernel is a counter example: any two distinct elements are mapped to elements at a mutual distance of 1), and not for every kernel the feature map is known. The following lemma bounds the geometry of the feature map using that of the kernel.

\begin{lemma}
Let $\kappa(x,y)=\left\langle\phi(x),\phi(y)\right\rangle$, and assume further that $\kappa$ fulfills a H\"older condition of order $\alpha$ uniformly in each parameter, that is,  $\left|\kappa(x,y)-\kappa(x+h,y)\right|\le L\norm{h}^\alpha$. Then $\phi$ uniformly fulfills a H\"older condition of order $\alpha/2$ with constant $\sqrt{2L}$.
\end{lemma}

This result is not sharp. For example, for the Gaussian case, both kernel and the feature map are H\"older order 1. 

\begin{proof}
Using the H\"older condition, we have that
 $\norm{\phi(x)-\phi(y)}_{\mc{H}}^{2} = \kappa\left(x,x\right)-\kappa\left(x,y\right)+\kappa\left(y,y\right)-\kappa\left(x,y\right)\le 2 L\norm{x-y}^\alpha$. 
All that remains is to take the square root of both sides.
\end{proof}

For a given feature mapping $\phi$, set of representations $\mc{R}$, we define two families of function classes so:

\begin{eqnarray*}
\mc W_{\phi,\lambda} & = & \left\{ h_{\phi,R_\lambda,D}:D\in\mc{D}^p\right\} \mbox{and}\\
\mc Q_{\phi,k,\delta} & = & \left\{ h_{\phi,H_k,D}:D\in\mc{D}^p\wedge\mu_{k-1}\left(\Phi D\right)\le \delta\right\}. \end{eqnarray*}

The next proposition completes this section by giving the cover number bounds for the representation error function classes induced by appropriate kernels, from which various generalization bounds easily follow, such as Theorem~\ref{thm:kernelGen}.

\begin{proposition}\label{prop:kernelCoverNumbers}
  Let $\mc{R}$ be a set of representations with a cover number bound of $\left(C/\e\right)^n$, and let either $\phi$ be uniformly $L$ H\"older condition of order $\alpha$ on $\mc{R}$, or $\kappa$ be uniformly $L$ H\"older of order $2\alpha$ on $\mc{R}$ in each parameter, and let $\gamma = \sup_{d\in\mc{R}}\norm{\phi(d)}_\mc{H}$.
  Then the function classes $\mc{W}_{\phi,\lambda}$ and $\mc{Q}_{\phi,k,\delta}$ taken as metric spaces with the supremum norm, have $\e$ covers of
  cardinalities at most 
$\left(C\left(\lambda \gamma L/\e\right)^{1/\alpha}\right)^{np}$
and
$\left(C\left(k \gamma^2 L/\left(\e\left(1-\delta\right)\right)\right)^{1/\alpha}\right)^{np}$, respectively.
\end{proposition}

\begin{proof}
We first consider the simpler case of $l_1$ constrained coefficients. If $\norm{a}_1\le\lambda$ and also $\max_{d\in\mc{D}}\norm{\phi(d)}_{\mc{H}}\le \gamma$ then by the considerations applied in section~\ref{sec:coverNumbers}, to obtain an $\e$ cover of the set $\left\{\min_{a}\norm{\left(\Phi D\right)a-\phi\left(x\right)}_{\mc{H}}:D\in\mc{D}\right\}$, it is enough to obtain an $\e/\left(\lambda\gamma\right)$ cover of $\left\{\Phi D:D\in \mc{D}\right\}$. If also $\phi$ is uniformly $L$ H\"older of order $\alpha$ over $\mc{R}$ then an $\left(\lambda\gamma L/\e\right)^{-1/\alpha}$ cover of the set of dictionaries is sufficient, which as we have seen requires at most $\left(C\left(\lambda\gamma L/\e\right)^{1/\alpha}\right)^{np}$ elements.

In the case of $l_0$ constrained representation, the bound on $\lambda$ due to Proposition~\ref{prop:bddCoefficients} is $\gamma k\left(1-\delta \right)$, and the result follows from the above by substitution.
\end{proof}

\section{Conclusions}

Our work has several implications on the design of dictionary learning
algorithms as used in signal, image, and natural language processing.  
First, the fact that generalization is only logarithmically dependent on the
 $l_1$ norm of the coefficient vector widens the set of applicable
approaches to penalization.
Second, in the particular case of $k$ sparse representation, we have shown that the Babel 
function is a key property
for the generalization of dictionaries. It might thus be useful to
modify dictionary learning algorithms so that they obtain dictionaries
with low Babel functions, possibly through regularization 
or through certain convex relaxations. Third, mistake bounds (e.g., \citealt{Mairal10})
on the quality of the solution to the coefficient finding optimization problem may lead to 
 generalization bounds for practical algorithms, by tying such algorithms to $k$ sparse representation.

The upper bounds presented here invite complementary lower bounds. The existing lower bounds for $k=1$ (vector quantization) and for $k=p$ (representation using PCA directions) are applicable, but do not capture the geometry of general $k$ sparse representation, and in particular do not clarify the effective dimension of the unrestricted class of dictionaries for it. We have not excluded the possibility that the class of unrestricted dictionaries has the same dimension as that of those with small Babel function. The best upper bound we know for the larger class, being the
trivial one of order $O\left(\binom{p}{k}n^2\right/m)$, leaves a 
significant gap for future exploration.

We mention also that the dependence on $\mu_{k-1}$ can also
be viewed from an ``algorithmic luckiness" perspective
\citep{Herbrich03Luckiness}: if the dictionary has favorable geometry
in the sense that the Babel function is small the generalization
bounds are encouraging. 

\section*{Acknowledgments}
We thank Shahar Mendelson for helpful discussions.
This research was partly supported by the European Community’s FP7-FET program, SMALL project, under grant agreement no. 225913.


\vskip 0.2in 

\bibliography{general}

\appendix

\section*{Appendix A: generalization with fast rates}\label{sec:fastrates}

In this appendix we justify some adaptations in Lemma~\ref{lem:gen} relative to its origins in~\cite{BBMlocalizedrademacher05}. Specifically, we assume only growth rates instead of combinatorial dimensions and give explicit constants for this case (no particular effort was made to make 
the constants tight). 

Following are some concepts and general results needed to prove Lemma~\ref{lem:gen} beyond those introduced in the main body of paper.

\begin{definition}
Let $F$ be a subset of a vector space $X$, $x\in X$. The star shaped closure of $F$ around $x$ is $$\star\left(F,x\right)=\left\{\lambda f+(1-\lambda)x : f\in F\wedge \lambda \in [0,1]\right\}.$$ 
\end{definition}

\begin{definition}
A function $f:\R_+ \to \R_+$ is called sub-root if it is non negative, non decreasing and if $r\mapsto f(r)/\sqrt(r)$ is non increasing for $r>0$.
\end{definition}

\begin{definition}
Let $\left\{ Z_{i}\right\} _{i=1}^{m}\cup\left\{ \e_{i}\right\} _{i=1}^{m}$
be independent variables, where $\e_{i}$ are uniform over $\left\{-1,1\right\}$,
and $Z_{i}$ are $i.i.d$. The empirical Rademacher average of $\mc F$
is \[
\hat{R}_{m}\left(\FF\right)=\E\left[\sup_{f\in\FF}\frac{1}{m}\sum_{i=1}^{n}\e_{i}f_{i}\left(Z_{i}\right)|Z_{1},\dots,Z_{m}\right].\]
\end{definition}
\begin{lemma}\label{lem:entropyInt}
Let $\hat{R}_{m}\left(\mc F\right)$ be the empirical Rademacher averages
of $\mc F$ on a sample $\left\{ Z_{i}\right\} _{i=1}^{m}$. We have
\[
\hat{R}_{m}\left(\mc F\right)\le12\int_{0}^{\infty}\sqrt{\frac{\log N\left(\e,\FF,L_{2}\left(\nu_{m}\right)\right)}{m}}d\e,\]
where $\nu_{m}=m^{-1}\sum_{i=1}^{m}\delta_{Z_{i}}$. 
\end{lemma}

See~\cite{learningNotesKakade08} for a proof.

\begin{lemma}\label{lem:logIntegral}
For any $\gamma\ge e^{\frac{1}{2}}$ and $x\in\left[0,1\right]$,
we have $\int_{0}^{x}\sqrt{\log\left(\gamma/\e\right)}d\e\le2x\sqrt{\log\left(\gamma/x\right)}$. \end{lemma}
\begin{proof}
The indefinite integral $\int_{0}^{x}\sqrt{\log\frac{\gamma}{\e}}d\e$
is $x\sqrt{\log\frac{\gamma}{x}}-\frac{\sqrt{\pi}}{2}\gamma\cdot\mbox{erf}\left(\sqrt{\log\frac{\gamma}{x}}\right)$,
where $\mbox{erf}(t)=2/\sqrt{\pi}\int_0^te^{-u^2}du$. Then  
\begin{align*}
\int_{0}^{x}\sqrt{\log\frac{\gamma}{\e}}d\e & =\left[x\sqrt{\log\frac{\gamma}{x}}-\frac{\sqrt{\pi}}{2}\gamma\mbox{erf}\left(\sqrt{\log\frac{\gamma}{x}}\right)\right]_{0}^{x}\\
& =x\sqrt{\log\frac{\gamma}{x}}-\frac{\sqrt{\pi}}{2}\gamma\mbox{erf}\left(\sqrt{\log\frac{\gamma}{x}}\right)-\lixz\left(x\sqrt{\log\frac{\gamma}{x}}-\frac{\sqrt{\pi}}{2}\gamma\mbox{erf}\left(\sqrt{\log\frac{\gamma}{x}}\right)\right)\\
& =x\sqrt{\log\frac{\gamma}{x}}-\frac{\sqrt{\pi}}{2}\gamma\left(\mbox{erf}\left(\sqrt{\log\frac{\gamma}{x}}\right)-\mbox{erf}\left(\infty\right)\right).\end{align*}

The error function $\mbox{erf}$ is related to the tail probability of a normal variable, also known as the $Q$ function. In particular,  $Q\left(x\right)=\frac{1}{2}\left(1-\mbox{erf}\left(\frac{x}{\sqrt{2}}\right)\right)\iff2Q\left(x\right)=1-\mbox{erf}\left(\frac{x}{\sqrt{2}}\right)\iff2Q\left(\sqrt{2}x\right)=1-\mbox{erf}\left(x\right)$. We thus substitute and then use the bound $Q(x)< e^{-x^2/2}/\left(x\sqrt{2\pi}\right)$: 

\begin{align*}
x\sqrt{\log\frac{\gamma}{x}}-\frac{\sqrt{\pi}}{2}\gamma\left(\mbox{erf}\left(\sqrt{\log\frac{\gamma}{x}}\right)-1\right) & =x\sqrt{\log\frac{\gamma}{x}}-\frac{\sqrt{\pi}}{2}\gamma\left(-2Q\left(\sqrt{2\log\frac{\gamma}{x}}\right)\right)\\
 & =x\sqrt{\log\frac{\gamma}{x}}+\sqrt{\pi}\gamma Q\left(\sqrt{2\log\frac{\gamma}{x}}\right)\\
 & \left(\mbox{Bound on }Q\right)\\
 & <x\sqrt{\log\frac{\gamma}{x}}+\sqrt{\pi}\gamma\frac{1}{\sqrt{2\log\frac{\gamma}{x}}\sqrt{2\pi}}e^{-\frac{\left(\sqrt{2\log\frac{\gamma}{x}}\right)^{2}}{2}}\\
 & =x\sqrt{\log\frac{\gamma}{x}}+x\frac{1}{2\sqrt{\log\frac{\gamma}{x}}}\\
 & =x\left(\sqrt{\log\frac{\gamma}{x}}+\frac{1}{2\sqrt{\log\frac{\gamma}{x}}}\right)\end{align*}
By our assumptions, $\frac{\gamma}{x}\ge\gamma\ge e^{1/2}\iff\sqrt{\log\frac{\gamma}{x}}>\sqrt{\frac{1}{2}}\iff\frac{1}{2\sqrt{\log\frac{\gamma}{x}}}<\sqrt{\frac{1}{2}}$,
then $x\left(\sqrt{\log\frac{\gamma}{x}}+\frac{1}{2\sqrt{\log\frac{\gamma}{x}}}\right)\le2x\sqrt{\log\frac{\gamma}{x}}$,
completing the proof.
\end{proof}

We return to prove Lemma~\ref{lem:gen}.

\begin{proof}
The core of the proof is to define particular sub-root function, and show its fixed point decays as  $1/m$. We then apply Theorem 3.3 of Bartlett, Bousquet and Mendelson~\cite{BBMlocalizedrademacher05} to this sub-root function to complete the proof.

We define the function
 $$\psi\left(r\right)=10\E R_{m}\left\{ f\in\star\left(\FF,0\right)|Ef^{2}\le r\right\} +\frac{11\log m}{m}.$$
By Lemma 3.4 of~\cite{BBMlocalizedrademacher05} (with $Tf=Ef^{2}$ and $\hat{f}=0$) and Lemma
3.2~\cite{BBMlocalizedrademacher05}, this function is sub-root and thus has a unique fixed point,
which we denote $r^{*}$, and $r<r^{*}\iff r<\psi\left(r\right)$.

To upper bound $r^*$ we first construct an upper bound on $\psi$, in which $Ef^2$ is replaced by $E_mf^2$, valid for $r\ge r^*$. The expectation of this upper bound is controlled using an entropy integral.

We make two observations.

\begin{enumerate} 

\item By Corollary 2.2 of~\cite{BBMlocalizedrademacher05}, with $b=1$, for $r>\psi(r)$ with probability at least $1-1/m$,
\[
\left\{ f\in\star\left(\FF,0\right):Ef^{2}\le r\right\} \subset\left\{ f\in\star\left(F,0\right):E_{m}f^{2}\le2r\right\} .
\]

\item By assumption $\left(\forall f\in\FF\right)\norm f_{L_{\infty}}\le1$
and this implies $R_{m}\left\{ f\in\star\left(\FF,0\right):Ef^{2}\le r\right\} \le1$.
\end{enumerate}
Combining the observations, we can bound $$\E R_{m}\left\{ f\in\star\left(\FF,0\right):Ef^{2}\le r\right\} \le\frac{1}{m}+\E R_{m}\left\{ f\in\star\left(F,0\right):E_{m}f^{2}\le2r\right\}.$$

Then $\psi\left(r\right)\le10\left(\frac{1}{m}+\E R_{m}\left\{ f\in\star\left(\FF,0\right):E_{m}f^{2}\le2r\right\} \right)+\frac{11\log\left(m\right)}{m}$,
and in particular
 $$r^{*}=\psi\left(r^{*}\right)\le10\left(\frac{1}{m}+\E R_{m}\left\{ f\in\star\left(\FF,0\right):E_{m}f^{2}\le2r^*\right\} \right)+\frac{11\log\left(m\right)}{m}.$$

We denote $\nu_{m}$ the empirical measure induced by the $m$ samples (whose expectation is $E_{m}$).
Under the metric $L_{2}\left(\nu_{m}\right)$, the set $\left\{ f\in\star\left(\FF,0\right):E_{m}f^{2}\le2r\right\} $ is covered by a single ball of radius $\sqrt{2r}$ around the zero function. Applying 
Lemma~\ref{lem:entropyInt}, we have
 \begin{align*}
\hat{R}_{m}\left\{ f\in\star\left(\FF,0\right):E_{m}f^{2}\le2r\right\}  & \le12\int_{0}^{\infty}\sqrt{\frac{\log N\left(\e,\star\left(\FF,0\right),L_{2}\left(\nu_{m}\right)\right)}{m}}d\e\\
& =12\int_{0}^{\sqrt{2r}}\sqrt{\frac{\log N\left(\e,\star\left(\FF,0\right),L_{2}\left(\nu_{m}\right)\right)}{m}}d\e.\end{align*}

Since $\left(\forall f\in\FF\right)\norm f_{L_{2}\left(\nu_{m}\right)}\le1$,
an $\e$ cover of $\mc F$ can be converted into an $\e$ cover of
$\star\left(\FF,0\right)$ by replacing each element by $1/\e+1$
balls on the segment from it to $0$. Then $N\left(\e,\star\left(\FF,0\right),L_{2}\left(\nu_{m}\right)\right)\le 2N\left(\e,\FF,L_{2}\left(\nu_{m}\right)\right)/\e$  $(a)$.
Using also the assumption on $C$ $(b)$ and Lemma~\ref{lem:logIntegral} $(c)$, we find 
\begin{align*}
12\int_{0}^{\sqrt{2r}}\sqrt{\frac{\log N\left(\e,\star\left(\FF,0\right),L_{2}\left(\nu_{m}\right)\right)}{m}}d\e & \stackrel{(a)}{\le}12\int_{0}^{\sqrt{2r}}\sqrt{\frac{\log\left(\left(\frac{C}{\e}\right)^{d}\frac{2}{\e}\right)}{m}}d\e\\
& \stackrel{(b)}{\le}12\sqrt{\frac{d+1}{m}}\int_{0}^{\sqrt{2r}}\sqrt{\log\left(\frac{C}{\e}\right)}d\e\\
& \stackrel{(c)}{\le}24\sqrt{\frac{2r\left(d+1\right)\log\left(\frac{C}{\sqrt{2r}}\right)}{m}}.\end{align*}

Substituting into $\psi\left(r^{*}\right)$, we obtain that 
\begin{align*}
r^{*}  \le & 10\left(\frac{1}{m}+24\sqrt{\frac{2r^{*}\left(d+1\right)\log\left(\frac{C}{\sqrt{2r^{*}}}\right)}{m}}\right)+\frac{11\log\left(m\right)}{m}\\ 
  = & 240\sqrt{\frac{2r^{*}\left(d+1\right)\log\left(\frac{C}{\sqrt{2r^{*}}}\right)}{m}}+\frac{10+11\log\left(m\right)}{m}.
\end{align*}

Let $\alpha>0$ be fixed. If $r^{*}\le\alpha C^{2}/2m$,
our first step is complete. If not, then $$r^{*}>\alpha C^{2}/2m\iff\sqrt{m/\alpha}>C/\sqrt{2r^{*}},$$
and then
 \begin{align*}
r^{*} & \le240\sqrt{\frac{2r^{*}\left(d+1\right)\log\left(\sqrt{\frac{m}{\alpha}}\right)}{m}}+\frac{10+11\log\left(m\right)}{m}\\
& =240\sqrt{\frac{r^{*}\left(d+1\right)\log\left(\frac{m}{\alpha}\right)}{m}}+\frac{10+11\log\left(m\right)}{m}\\
& \le2\max\left\{ 240\sqrt{\frac{r^{*}\left(d+1\right)\log\left(\frac{m}{\alpha}\right)}{m}},\frac{10+11\log\left(m\right)}{m}\right\} .\end{align*}
Then either $r^{*}\le\left(20+22\log\left(m\right)\right)/m$ (and the
first step is complete), or $$r^{*}\le480\sqrt{\left(r^{*}\left(d+1\right)\log\left(m/\alpha\right)\right)/m}\iff r^{*}\le\left(480\right)^{2}\left(\left(d+1\right)\log\left(m/\alpha\right)\right)/m$$
and again we are done. We conclude that 
\[
r^{*}\le\max\left\{ \frac{\alpha C^{2}}{2m},\left(480\right)^{2}\frac{\left(d+1\right)\log\left(\frac{m}{\alpha}\right)}{m},\frac{20+22\log\left(m\right)}{m}\right\}. \]

Having proved $r^*$ decays approximately as $1/m$, we apply Theorem 3.3 of ~\cite{BBMlocalizedrademacher05}, with
 $a=0;b=1;B=1;Tf=Ef^{2}$. 
By definition of $T$, it is clear that 
\begin{align*}
\psi\left(r\right) & =10\E R_{m}\left\{ f\in\star\left(\FF,0\right)|Ef^{2}\le r\right\} +\frac{11\log m}{m}\\
& \ge\E R_{m}\left\{ f\in\star\left(\FF,0\right)|Ef^{2}\le r\right\} \\
& =\E R_{m}\left\{ f\in\star\left(\FF,0\right)|Tf\le r\right\} \end{align*}
holds, then we can use part 2 of Theorem 3.3 of~\cite{BBMlocalizedrademacher05}, which allows the
conclusion that for all $f\in\mc F$, $Ef\le\frac{K}{K-1}E_{m}f+6Kr^{*}+\frac{11x+5K}{m}$.
\end{proof}

\end{document}